\documentclass{article}

\PassOptionsToPackage{numbers,sort&compress}{natbib}

 \usepackage[preprint]{neurips_2026}

\usepackage[utf8]{inputenc} 
\usepackage[T1]{fontenc}    
\usepackage{hyperref}       
\usepackage{url}            
\usepackage{booktabs}       
\usepackage{amsfonts}       
\usepackage{nicefrac}       
\usepackage{microtype}      
\usepackage{xcolor}         
\usepackage{booktabs,tabularx,array}
\definecolor{ceruleanblue}{rgb}{0.16, 0.32, 0.75}

\hypersetup{
 colorlinks=true,
 citecolor=ceruleanblue,
 linkcolor=ceruleanblue,
 urlcolor=black}

\usepackage{algorithm}
\usepackage{algpseudocode}
\usepackage{multirow}
\usepackage{amsmath}
\usepackage{amssymb}
\usepackage{mathtools}
\usepackage{amsthm}

\usepackage{booktabs}
\usepackage{multirow}
\usepackage{rotating}
\usepackage{adjustbox}
\usepackage[table]{xcolor}

\usepackage{amsthm}
\newtheorem{proposition}{Proposition}

\usepackage{wrapfig}
\usepackage{algorithmicx}
\usepackage{algpseudocode}
\usepackage{caption}

\algrenewcommand\algorithmicrequire{\textbf{Input:}}
\algrenewcommand\algorithmicensure{\textbf{Output:}}

\usepackage{graphicx}
\usepackage{arydshln}  

\usepackage[capitalize,noabbrev]{cleveref} 

\usepackage{bbm}        
\usepackage{colortbl}   
\usepackage{pifont}     
\usepackage{placeins}
\usepackage{bbding} 

\usepackage[most]{tcolorbox}
\usepackage{xcolor}
\usepackage{dashrule}

\definecolor{caseblue}{RGB}{47,95,166}
\definecolor{casebg}{RGB}{255,255,255}

\newtcolorbox{casebox}[1]{
  enhanced,
  breakable,
  colback=casebg,
  colframe=caseblue,
  colbacktitle=caseblue,
  coltitle=white,
  title={#1},
  fonttitle=\bfseries\fontsize{11}{13}\selectfont,
  fontupper=\normalsize,
  boxrule=1pt,
  arc=3mm,
  outer arc=3mm,
  left=12pt,
  right=12pt,
  top=10pt,
  bottom=10pt,
  lefttitle=10pt,
  righttitle=10pt,
  toptitle=4pt,
  bottomtitle=4pt,
  drop fuzzy shadow=black!18!white,
  before skip=10pt,
  after skip=14pt
}

\usepackage[utf8]{inputenc} 
\usepackage[T1]{fontenc}    
\usepackage{url}            
\usepackage{booktabs}       
\usepackage{amsfonts}       
\usepackage{nicefrac}       
\usepackage{microtype}      

 \usepackage{graphicx}
\usepackage{subfigure} 
\usepackage{pifont}
\usepackage{adjustbox}
\usepackage{lipsum}
\usepackage{color}
\usepackage{wrapfig}
\usepackage{booktabs}
\usepackage[textsize=scriptsize]{todonotes}
\usepackage{multirow,mathtools}
\usepackage{threeparttable}
\usepackage{tikz-cd}

\definecolor{lightblue}{rgb}{0.68, 0.85, 0.9}
\definecolor{lightgreen}{rgb}{0.56, 0.93, 0.56}
\definecolor{lightskyblue}{rgb}{0.53, 0.81, 0.98}
\definecolor{non-photoblue}{rgb}{0.64, 0.87, 0.93}
\definecolor{magicmint}{rgb}{0.67, 0.94, 0.82}
\definecolor{mossgreen}{rgb}{0.68, 0.87, 0.68}
\definecolor{salmon}{rgb}{1.0, 0.55, 0.41}
\definecolor{babypink}{rgb}{0.96, 0.76, 0.76}
\definecolor{darkgreen}{rgb}{0, 0.7, 0}

\DeclareMathAlphabet\mathbfcal{OMS}{cmsy}{b}{n}

\usepackage{color, colortbl}
\definecolor{Gray}{gray}{0.93}
\definecolor{Orange}{rgb}{1,0.5,0}
\definecolor{DGray}{gray}{0.83}
\definecolor{LightCyan}{rgb}{0.88,1,1}

\usepackage{dsfont}
\usepackage{enumerate}
\usepackage{graphicx}
\usepackage{threeparttable}

\definecolor{Red}{rgb}{0.6,0,0}
\definecolor{Blue}{rgb}{0,0,0.8}
\definecolor{Green}{rgb}{0,0.6,0.9}
\definecolor{airforceblue}{rgb}{0.36, 0.54, 0.66}
\definecolor{ao(english)}{rgb}{0.0, 0.5, 0.0}
\definecolor{azure(colorwheel)}{rgb}{0.0, 0.5, 1.0}
\definecolor{crimson}{rgb}{0.86, 0.08, 0.24}
\definecolor{darkcerulean}{rgb}{0.03, 0.27, 0.49}
\definecolor{cobalt}{rgb}{0.0, 0.28, 0.67}
\definecolor{rosegold}{rgb}{0.72, 0.43, 0.47}
\definecolor{orange-red}{rgb}{1.0, 0.27, 0.0}
\definecolor{mountainmeadow}{rgb}{0.19, 0.73, 0.56}
\definecolor{malachite}{rgb}{0.04, 0.85, 0.32}
\definecolor{darkblue}{rgb}{0.0, 0.0, 0.55}
\definecolor{customred}{rgb}{1, 0.85, 0.85}
\definecolor{customcitecolor}{rgb}{0.7, 0.5, 1}
\definecolor{custompink}{rgb}{0.8, 0.3, 0.3}
\definecolor{customgreen}{rgb}{0.3, 0.8, 0.3}
\definecolor{Lightgreen}{rgb}{0.8, 1, 0.9}
\definecolor{LightCyan}{rgb}{0.8, 0.9, 1}
\definecolor{mygray}{gray}{0.8}

\usepackage[most]{tcolorbox}
\newtcolorbox{mybox}[2][]{%
  attach boxed title to top center
               = {yshift=-8pt},
  colback      = Gray,
  colframe     = black,
  fonttitle    = \bfseries,
  colbacktitle = white,
  title        = #2,#1,
  enhanced,
}

\DeclarePairedDelimiterX{\inp}[2]{\langle}{\rangle}{#1, #2}

\makeatletter
\newcommand*{\rom}[1]{\expandafter\@slowromancap\romannumeral #1@}
\makeatother
\newcommand{\mycomment}[1]{}

\usepackage{pifont}

\usepackage{color, colortbl}
\definecolor{Gray}{gray}{0.93}
\definecolor{Orange}{rgb}{1,0.5,0}
\definecolor{DGray}{gray}{0.83}
\definecolor{LightCyan}{rgb}{0.88,1,1}

\definecolor{WarnREd}{rgb}{1,0.4,0.4}
\definecolor{WarnOrange}{rgb}{1,0.682,0.502}
\definecolor{WarnPink}{rgb}{0.9176, 0.7215, 0.7215}
\definecolor{GoodGreen}{rgb}{0.5019, 0.9215, 0.6039}
\definecolor{deepgreen}{RGB}{0,128,0}   
\definecolor{ochreyellow}{RGB}{204,119,34} 
\definecolor{GoodGreen}{rgb}{0.5019, 0.9215, 0.6039}
\definecolor{NiceYellow}{rgb}{0.98, 0.92, 0.60}

\usepackage[T1]{fontenc}

\definecolor{styleblue}{HTML}{504099}
\definecolor{mypurple}{HTML}{9391ff}
\usepackage{wrapfig}

\usepackage{multirow,mathtools }

\usepackage{pifont}
\usepackage{color, colortbl}

\usepackage{blindtext}
\usepackage{lipsum}

\usepackage{multirow}
\usepackage{listings}

\usepackage{bbm}

\usepackage [english]{babel}
\usepackage [autostyle, english = american]{csquotes}
\usepackage[nottoc]{tocbibind}

\usepackage{pifont}
\usepackage{url}
\usepackage[most]{tcolorbox}

\usepackage{lipsum}
\usepackage{soul}
\usepackage{xcolor}
\usepackage{wrapfig}
\usepackage{multirow,mathtools } 

\usepackage{adjustbox}
\MakeOuterQuote{"}

\usepackage{microtype}
\usepackage{graphicx}
\usepackage{booktabs} 

\usepackage{amsmath}
\usepackage[capitalize,noabbrev]{cleveref}
\usepackage{tablefootnote}

\usepackage{amsmath,amsfonts,bm,nicefrac}

\def\eqref#1{Eq.~\ref{#1}}

\def\1{\bm{1}}

\DeclareMathAlphabet{\mathsfit}{\encodingdefault}{\sfdefault}{m}{sl}
\SetMathAlphabet{\mathsfit}{bold}{\encodingdefault}{\sfdefault}{bx}{n}

\newcommand{\be}{\begin{eqnarray} \begin{aligned}}
\newcommand{\ee}{\end{aligned} \end{eqnarray} }
\newcommand{\benn}{\begin{eqnarray*} \begin{aligned}}
\newcommand{\eenn}{\end{aligned} \end{eqnarray*} }

\title{SAGE: Retain-Aware Post-Hoc Sanitization \\ of Final Unlearning Vector}

\author{%
\begin{tabular}{c}
Jingyuan Zhang\textsuperscript{1,\textdagger} \quad
Yucheng Bai\textsuperscript{1,\textdagger} \quad
Peixi Wen\textsuperscript{1} \quad
Zhehao Huang\textsuperscript{1} \\
Zhengbao He\textsuperscript{1} \quad
Hanling Tian\textsuperscript{1} \quad
Xinwen Cheng\textsuperscript{1} \quad
Haiyin Ran\textsuperscript{1} \quad
Xiaolin Huang\textsuperscript{1,\Envelope} \\
\textsuperscript{1}Institute of Image Processing and Pattern Recognition, Shanghai Jiao Tong University \\
\end{tabular}
}

\makeatletter
\renewcommand{\@noticestring}{%
Preprint. \quad
\textsuperscript{\textdagger}Equal contribution \quad
\textsuperscript{\Envelope}Corresponding author}
\makeatother

\begin{document}

\maketitle

\begin{abstract}

Large Language Model (LLM) unlearning aims to remove undesirable knowledge or behaviors while preserving retained capabilities. 
Current unlearning methods all involve a trade-off between unlearning and retention. We have found that the retention activation bias can also be used to quantify the damage an unlearning method inflicts on retention, without considering the specific implementation of the unlearning process. This allows us to restore retention performance for any unlearning method using a post-hoc approach.
Therefore, we propose a complementary post-hoc setting to sanitize the final update vector without rerunning the original unlearning pipeline. In this setting, we design \textbf{SAGE}, \textbf{S}pectral \textbf{A}ctivation-\textbf{GE}ometry Sanitization, a source-agnostic correction for final unlearning updates. SAGE collects real module inputs from a small retain proxy, extracts their dominant activation geometry, and solves a source-anchored optimization objective in closed form, which suppresses update components aligned with high-energy retained directions while preserving the source method's forgetting carrier.
Across multiple unlearning methods, model scales, and benchmarks, SAGE consistently relieves the retain--forget trade-off, identifying post-hoc sanitization of final vectors as a practical and underexplored axis for machine unlearning.

\end{abstract}

\section{Introduction}
\label{sec:intro}

\begin{wrapfigure}{r}{0.50\linewidth}
    \vspace{-3.5\baselineskip}
    \centering
    \captionsetup{font=footnotesize}
    \includegraphics[width=\linewidth]{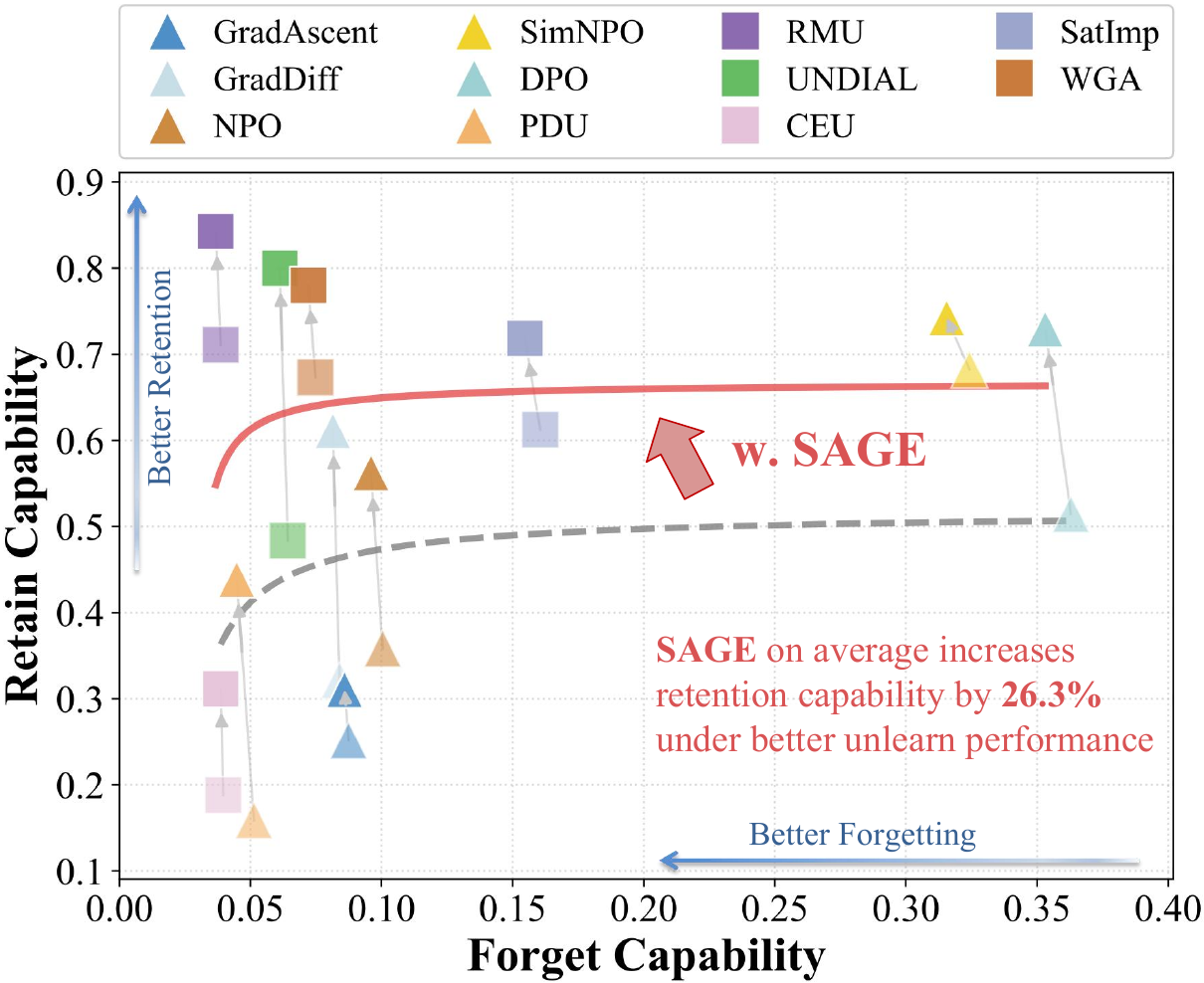}
    \caption{Overall Performance. Triangles denote loss-driven methods and squares denote constraint-guided methods. Points are averaged over multiple models and forgetting difficulties. Light markers and gray line fit original baselines; dark markers and red line fit methods \textbf{with SAGE}.}
    \label{fig:parato_main}
    \vspace{-0.8\baselineskip}
\end{wrapfigure}

Large language models (LLMs)~\citep{bai2023qwen,grattafiori2024llama3herd,touvron2023llama,wolf2020transformers,yang2025qwen3} can memorize private information, copyrighted content, harmful behaviors, and other undesirable data patterns from pretraining and fine-tuning corpora~\citep{nicholas2021extracting,liu2024rethinking,wang2025wga,jang2022knowledge,yao2024llmunlearning}. 
This has made machine unlearning an increasingly important problem for safety, privacy, and regulatory compliance~\citep{li2024wmdp,yao2024surveyllm,zhang2024rightforgotten}. 
While exact retraining remains the cleanest deletion target in classical machine unlearning, it is often computationally prohibitive at modern model scales, motivating efficient approximate alternatives~\citep{bourtoule2020mu,neel2020descenttodelete}. 
For LLMs, it is more challenging because knowledge is encoded in highly distributed representations learned from large-scale corpora, making it difficult to remove targeted information without inducing widespread collateral damage to general capabilities~\citep{brown2020languagemodelsfewshotlearners,touvron2023llama}.

Recent work has produced a wide range of approximate LLM unlearning methods, including gradient and objective-based~\citep {maini2024TOFU,wang2024FLAT,fan2025SimNPO,entesari2025PDU,yang2025CEU,zhang2024NPO}, preference-style optimization~\citep{dong2024UNDIAL,mekala2024AltPO}, representation-level or localized-parameter interventions~\citep{li2024wmdp,shen2025LUNAR}, loss-reweighting methods~\citep{wang2025wga,yang2025SatImp}, and task-vector methods~\citep{ilharco2023taskarithmetic,dong2025TSA,cai2026perta,kim2025negmerge}. 
Furthermore, benchmarks such as TOFU~\citep{maini2024TOFU}, MUSE~\citep{shi2024MUSE}, and WMDP~\citep{li2024wmdp} have made clear that forgetting must be balanced against leakage, utility, and retained behavior.
Based on them, many lightweight interventions, ranging from inference-time control~\citep{huang2025offset,pawelczyk2024ICU,Liu2024ECO,bhaila2024SPUL} to training-time plug-ins~\citep{wang2025GRU,zhou2026GU}, have also emerged. 
Although they have already shown that adding a lightweight correction mechanism can be an effective way to further improve unlearning, most of them remain tightly coupled to specific stages of the original pipeline: training-time plug-ins must be incorporated during optimization and rely on rerunning or modifying the original unlearning process.
Part (d) in Figure~\ref{fig:overview} shows that during source unlearning, retain activation bias rises sharply while retained ability drops substantially, even after the forgetting strength has largely stabilized. 
This suggests that the final unlearning update is still not fully retain-aware: the deployed update continues to contain components that disproportionately perturb retained activation directions, indicating that there remains meaningful room to improve retention by sanitizing the completed update itself.
This motivates a decoupled post-hoc setting, in which the object of correction is no longer the training dynamics, but the final vector itself. 
We try to construct a sanitized vector to better trade off forgetting and retention, without revisiting the original unlearning process.

\begin{figure}[t]
    \centering
    \includegraphics[width=1.0\linewidth]{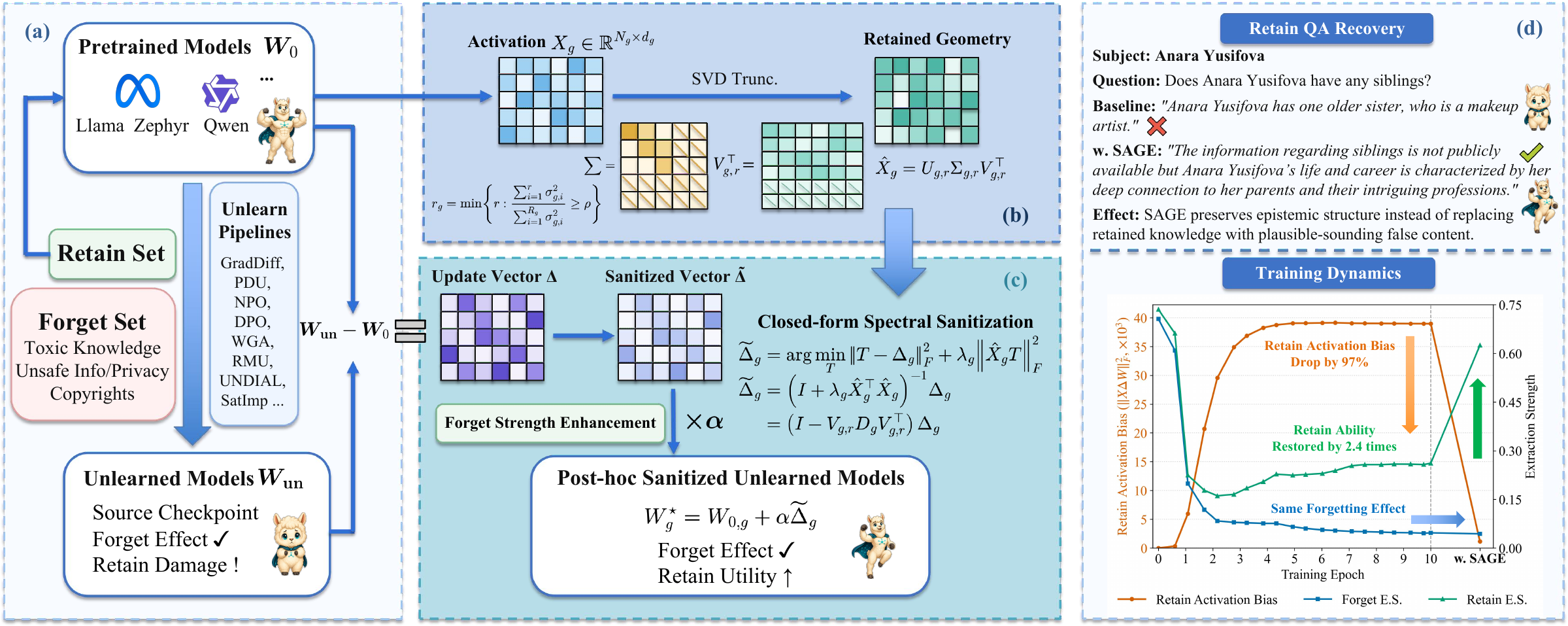}
    \caption{
    Overview of the proposed SAGE method: (a) Acquire post-hoc update vector from unlearning baselines; (b) Construct stable and denoised retain geometry from module-level input activations; (c) Apply closed-form spectral sanitization and amplifier to get unlearned models ; (d) Demonstration of retain restoration on a QA sample and drop on retain activation bias.
    }
    \label{fig:overview}
    \vspace{-2.0\baselineskip}
\end{figure}

Therefore, we propose \textbf{SAGE}, \textbf{S}pectral \textbf{A}ctivation-\textbf{GE}ometry Sanitization, a post-hoc sanitizer for final unlearning vectors.
To construct a retained activation basis, we sample a small retain calibration set and collect module-level input activations through no-gradient forward passes.
To reduce the output perturbation by minimizing the update response energy on the retained input geometry, we apply truncated Singular Value Decomposition (SVD) to identify a stable and denoised low-rank subspace.
As the source unlearning method has already been optimized for the forget objective, its final parameter displacement is the most direct forgetting carrier available in the post-hoc setting.
SAGE therefore performs source-anchored sanitization: it keeps the sanitized update close to the original source update to preserve the learned forgetting signal, while suppressing the components that produce large retained-activation responses.
In the objective's closed-form solution, directions with larger retained singular energy are attenuated more strongly, whereas update components outside the dominant retained subspace are largely preserved.
Finally, SAGE applies an amplifier to ensure effective forgetting, following the task-vector view of model editing and unlearning, where the direction of a parameter displacement determines the type of behavioral change and its magnitude controls the strength of that change~\citep{ilharco2023taskarithmetic,cai2026perta}.

Empirically, SAGE consistently improves the retain-forget trade-off across diverse source unlearning methods, model scales, and forget ratios, without rerunning the original unlearning pipeline. On TOFU, it increases average retention capability by 26.3\% with better unlearn effect. These improvements remain stable from 1B to 8B models and become more pronounced as the amount of content to be forgotten increases. What's more, SAGE improves model utility by 2.2\% and reduces privacy leakage by 6.2\% on average. 
Its benefits further transfer beyond structured QA-style forgetting. 
SAGE improves retention capabilities by about 39.8\% and 5.2\% on MUSE and WMDP-cyber respectively, while largely preserving forgetting behavior. 
Notably, for TOFU dataset, SAGE uses roughly 3\% of the full retain set, and remains robust even under smaller proxy budgets. 
Together, these results suggest that post-hoc sanitization of final unlearning updates is a practical and underexplored design axis for machine unlearning.
Our main contributions are as follows:

\begin{itemize}
    \item We study a practical post-hoc unlearning setting to reduce retain-side collateral damage in which the final parameter update of a source unlearning method is sanitized after training, without rerunning or modifying the original pipeline.

    \item We propose \textbf{SAGE}, a module-wise closed-form sanitizer that builds retained activation geometry from a small retain set and applies a singular-value-aware soft spectral operator to suppress retained-response directions while preserving forgetting effect.

    \item Across multiple source unlearning methods, model scales, and benchmarks, we show that SAGE consistently improves the retention capability, while also improving utility and privacy leakage, with less time and computation resources.
\end{itemize}

\section{Related Work}

\subsection{LLM Unlearning}

Machine unlearning for LLMs is motivated by the need to remove private, copyrighted, or hazardous knowledge without retraining from scratch~\citep{jang2022knowledge,yao2024llmunlearning}. Since exact deletion is generally infeasible for modern LLMs, recent work has focused on approximate unlearning together with evaluation protocols that jointly measure removal and preservation. Benchmarks such as TOFU~\citep{maini2024TOFU}, WMDP~\citep{li2024wmdp}, MUSE~\citep{shi2024MUSE} and OpenUnlearning~\citep{dorna2025openunlearning} have made the forget--retain trade-off a central concern, showing that effective unlearning should evaluate together with retained behavior, utility and leakage.

Most existing methods improve this trade-off during optimization. Gradient- and objective-based approaches modify the forget objective directly, including gradient ascent~\citep{maini2024TOFU}, NPO~\citep{zhang2024NPO}, and SimNPO~\citep{fan2025SimNPO}, while representation-level methods such as RMU~\citep{li2024wmdp} and LUNAR~\citep{shen2025LUNAR} intervene on internal activations or hidden states to suppress unwanted knowledge. More recently, lightweight retain-aware plug-ins have been proposed: GRU~\citep{wang2025GRU} rectifies retention-damaging update directions during training, and GU~\citep{zhou2026GU} removes components aligned with the retain-gradient subspace. Other lightweight methods, such as offset unlearning~\citep{huang2025offset}, in-context unlearning~\citep{pawelczyk2024ICU}, and soft prompting~\citep{bhaila2024SPUL}, reduce access requirements by acting at inference time rather than on model weights. In contrast, SAGE operates after training is complete and directly sanitizes the final deployed weight update, without modifying the original optimization loop.

\subsection{Weight-Space Unlearning and Preservation}

Our work is also related to methods that treat parameter differences as editable weight-space objects. Task Arithmetic~\citep{ilharco2023taskarithmetic} shows that fine-tuning deltas can act as task vectors whose addition, subtraction, or scaling steers model behavior without further training. However, single-vector unlearning is sensitive to fine-tuning configuration, scaling, and candidate selection. Task Simplex Arithmetic~\citep{dong2025TSA} and NegMerge~\citep{kim2025negmerge} improve robustness through multi-vector aggregation. PerTA~\citep{cai2026perta} uses gradient or diagonal-Fisher estimates to rescale task vectors, addressing over-forgetting. These methods improve vector selection, merging, or parameter-wise merging to acquire a stable or retain-friendly unlearning vector. SAGE instead accepts final vectors and sanitizes it using module-level retain activation geometry.

A related preservation-oriented literature comes from model editing. ROME~\citep{meng2022ROME} edits factual associations through localized rank-one updates, and MEMIT~\citep{meng2023MEMIT} extends this paradigm to many edits by distributing updates across layers. AlphaEdit~\citep{fang2024AlphaEdit} reduces collateral damage by projecting edit perturbations into the null space of preserved knowledge keys. This literature shares our emphasis on protecting unaffected knowledge, but it mainly targets factual editing and often relies on hard preservation constraints.
By contrast, unlearning produces a broader final update that must reduce unsafe knowledge without unnecessarily harming retained behavior. SAGE brings this preservation perspective into unlearning by correcting the final update vector itself, reducing the components that disproportionately disturb retained behavior while preserving the original unlearning effect.

\section{Method}
\label{sec:method}

\subsection{Problem Setup}
\label{sec:method_setup}

Existing unlearning methods often differ substantially in their training objectives, optimization procedures, and access assumptions, yet they share a common challenge: improving forgetting typically comes at the cost of retained capabilities. 
As shown in Figure~\ref{fig:overview}, retain-side collateral damage can be reflected by retain activation bias, independent of implementation details of unlearning algorithms.
Therefore, we propose a post-hoc final-update setting for LLM unlearning, providing a global correction targeted at accumulated updates instead of step-wise training-time control.

Given a base model \(W_0\) and a source unlearned model \(W_{\mathrm{un}}\), we get the final vector from the pretrained model to the unlearned model.
Rather than editing all parameters uniformly, we restrict SAGE to a structured set of editable modules \(\mathcal{G}\), comprising the attention and MLP projection matrices, as they admit a meaningful input-side activation geometry. For each module \(g\), we define source update as
\begin{equation}
\Delta_g=
\begin{cases}
W_{\mathrm{un},g}-W_{0,g}, & g\in\mathcal{G},\\[3pt]
0, & g\notin\mathcal{G}.
\end{cases}
\label{eq:editable_delta}
\end{equation}

\subsection{Retained Activation Geometry}
\label{sec:method_geometry}

Retain-side damage is not determined solely by the norm of the source update. More importantly, it depends on whether the update acts strongly along the dominant input directions associated with retained capabilities. To capture this structure, we run a no-gradient forward pass on a small retain calibration proxy \(D_{\mathrm{cal}}\) for each editable module \(g\) and collect the module's real input activations into $X_g \in \mathbb{R}^{N_g \times d_g}$, where \(d_g\) is the module input dimension and \(N_g\) is the total number of collected token-level inputs.

We then compute the singular value decomposition $X_g = U_g \Sigma_g V_g^\top$, and retain the top \(r_g\) singular directions according to the cumulative-energy criterion. The resulting truncated operator emphasizes dominant retained directions while discarding low-energy and probable noisy components, providing a stable geometry for the post-hoc sanitizer and avoiding overfitting. This yields a truncated retained-geometry operator
\begin{equation}
\hat X_g := U_{g,r}\Sigma_{g,r}V_{g,r}^\top,
\qquad
\hat C_g := \hat X_g^\top \hat X_g
= V_{g,r}\Sigma_{g,r}^2V_{g,r}^\top.
\label{eq:denoised_geometry}
\end{equation}
Here, \(\hat X_g\) captures the dominant retained input geometry of module \(g\), while \(\hat C_g\) is the corresponding input-side Gram operator. In particular, \(\hat C_g\) encodes the relative importance of retained input directions and will serve as the geometry-aware weighting operator in the sanitizer below.

\subsection{Closed-Form Spectral Sanitization}
\label{sec:method_sanitization}

In the post-hoc setting, the source method’s final parameter displacement is the only directly available carrier of the forgetting effect achieved during unlearning. Accordingly, the sanitizer should remain close to the source update and only correct the part that is most harmful to retained behaviors.

To this end, for each module g, SAGE optimizes a source-anchored objective with two complementary terms: a proximity term that keeps the sanitized update close to the source displacement, and a response penalty that suppresses update directions introducing output perturbation on retained activation geometry. For module \(g\), let \(m_g\) denote the output dimension, and write the corresponding update in operator form as \(\Delta_g \in \mathbb{R}^{d_g \times m_g}\). Concretely, we solve

\begin{equation}
\widetilde{\Delta}_g
=
\arg\min_T
\left\|T-\Delta_g\right\|_F^2
+
\lambda_g \left\|\hat X_g T\right\|_F^2 .
\label{eq:main_objective}
\end{equation}

The first term preserves the learned forgetting signal already encoded in the source update, while the second term penalizes retain-harming directions. Without the source anchor, minimizing only the response term would collapse to the trivial zero update.

\begin{wrapfigure}{r}{0.50\linewidth}
    \vspace{-0.8\baselineskip}
    \footnotesize
    \noindent\rule{\linewidth}{0.8pt}
    
    \vspace{2pt}
    \refstepcounter{algorithm}
    \noindent\textbf{Algorithm \thealgorithm} SAGE Framework
    \label{alg:sage}
    
    \noindent\rule{\linewidth}{0.5pt}
    
    \begin{algorithmic}[1]
        \State \textbf{Input:} base model \(W_0\), source model \(W_{\mathrm{un}}\), editable modules \(\mathcal{G}\), retain calibration proxy \(D_{\mathrm{cal}}\), regularization strengths \(\{\lambda_g\}_{g\in\mathcal{G}}\), energy threshold \(\rho\), calibrated \(\alpha\)
        \State \(\Delta \gets W_{\mathrm{un}} - W_0\)
        \For{each \(g \in \mathcal{G}\)}
            \State collect retained module inputs \(X_g\) from \(D_{\mathrm{cal}}\)
            \State compute \(X_g = U_g \Sigma_g V_g^\top\)
            \State choose \(r_g\) by energy threshold \(\rho\)
            \State \(\hat X_g \gets U_{g,r}\Sigma_{g,r}V_{g,r}^\top\)
            \State \(\hat C_g \gets \hat X_g^\top \hat X_g\)
            \State \(\widetilde{\Delta}_g \gets (I - V_{g,r} D_g V_{g,r}^\top)\Delta_g\)
        \EndFor
        \State choose \(\alpha\) by matched-forgetting calibration
        \For{each module \(g\)}
            \If{\(g \in \mathcal{G}\)}
                \State \(W_g^\star \gets W_{0,g} + \alpha \widetilde{\Delta}_g\)
            \Else
                \State \(W_g^\star \gets W_{0,g}\)
            \EndIf
        \EndFor
        \State \textbf{Return} \(W^\star\)
    \end{algorithmic}
    
    \noindent\rule{\linewidth}{0.5pt}
    \vspace{-4\baselineskip}
\end{wrapfigure}

\begin{proposition}[Unique closed-form sanitizer]
\label{prop:closed_form}
For each editable module \(g\) and any \(\lambda_g \ge 0\), the objective in ~\eqref{eq:main_objective} is strongly convex in \(T\) and admits the unique minimizer
\begin{equation}
\widetilde{\Delta}_g
=
\left(I+\lambda_g \hat C_g\right)^{-1}\Delta_g
=
\left(I - V_{g,r} D_g V_{g,r}^\top\right)\Delta_g,
~D_g
=
\mathrm{diag}\!\left(
\frac{\lambda_g \sigma_{g,1}^2}{1+\lambda_g \sigma_{g,1}^2},
\dots,
\frac{\lambda_g \sigma_{g,r}^2}{1+\lambda_g \sigma_{g,r}^2}
\right).
\label{eq:closed_form}
\end{equation}
\end{proposition}

\noindent\textit{Proof sketch.}
~\eqref{eq:main_objective} is a strictly convex quadratic in \(T\), with Hessian
\(2(I+\lambda_g \hat C_g)\otimes I_{m_g}\).
Since \(\hat C_g \succeq 0\), the Hessian is positive definite. Setting the gradient to zero and applying Woodbury matrix identity yields ~\eqref{eq:closed_form}.

This closed-form solution reveals that SAGE acts as a soft spectral sanitizer on the source update. Along each \(v_{g,i}\) retained principal directions, the source update is shrunk by \(\frac{1}{1+\lambda_g \sigma_{g,i}^2}\), which is determined by the retained singular energy, and directions orthogonal to \(\mathrm{span}(V_{g,r})\) is left unchanged. So directions with larger retained energy are attenuated more strongly, and components outside the dominant retained subspace are largely preserved.

As a result, SAGE is not a hard null-space projection that indiscriminately removes all retained-subspace components as \(\lambda_g \to \infty\). Instead, it performs a continuous, geometry-aware shrinkage that suppresses the most retain-sensitive directions while preserving as much of the source forgetting carrier as possible. Also, this operator does not amplify the retained-geometry response measured on the calibration proxy, reducing the retain-side disturbance.

\paragraph{Final Composition and Forget-Matched Calibration}
As the magnitude of sanitized unlearning vectors controls the strength of forgetting~\citep{ilharco2023taskarithmetic}, we apply an amplifier and form the final model as
\begin{equation}
W_g^\star=W_{0,g}+\alpha\,\widetilde{\Delta}_g,
\end{equation}
where \(\alpha\) is a scaling coefficient selected by grid search on a disjoint forget-side calibration subset.

\paragraph{Mechanistic Discussion: Differential Suppression.}
To further interpret when post-hoc sanitization is beneficial, we compare the relative suppression induced by SAGE on retain and forget activations. For module \(g\), define
\begin{equation}
S_g^{(r)} := 1 - \frac{\left\|X_g^{(r)} \widetilde{\Delta}_g\right\|_F^2}{\left\|X_g^{(r)} \Delta_g\right\|_F^2},
\qquad
S_g^{(f)} := 1 - \frac{\left\|X_g^{(f)} \widetilde{\Delta}_g\right\|_F^2}{\left\|X_g^{(f)} \Delta_g\right\|_F^2},
\end{equation}
and let
\begin{equation}
\Gamma_g := S_g^{(r)} - S_g^{(f)}.
\end{equation}
A positive \(\Gamma_g\) indicates that SAGE suppresses retain-side response more strongly than forget-side response, which helps explain why, after forget-matched calibration, the sanitized update can improve retention while preserving forgetting behavior. We study this quantity empirically in Section ~\ref{sec:sd_analysis}.

\section{Experiments}
\label{sec:experiments}

\subsection{Experiment Setups}

\begin{table*}[t]
\centering
\small
\definecolor{sageblue}{RGB}{235,243,248}
\definecolor{sagesup}{RGB}{230,138,30}
\newcommand{\gain}[1]{\textsuperscript{\textcolor{sagesup}{\bfseries$\boldsymbol{\uparrow}\mkern-10mu$#1}}}
\newcommand{\loss}[1]{\textsuperscript{\textcolor{sagesup}{\bfseries-#1}}}
\setlength{\tabcolsep}{3.5pt}
\renewcommand{\arraystretch}{1.08}
\caption{Main TOFU results across Llama-3-1B,3B,8B on Forget-1\%,5\%,10\% split, with retain/unlearn Extraction strength (ES), absolute privacy leak ($|\mathrm{Priv.\ Leak}|$), and model utility (MU). Better results with SAGE are highlighted in \textbf{bold}, and retain-capability gains are shown in \textbf{\textcolor{sagesup}{orange}}.}
\label{tab:main_tofu}
\begin{adjustbox}{max width=\textwidth}
\begin{tabular}{lcccc!{\vrule width 0.6pt}cccc!{\vrule width 0.6pt}cccc}
\toprule
\multirow{2}{*}{Method} & \multicolumn{4}{c}{\textbf{Forget-1\%}} & \multicolumn{4}{c}{\textbf{Forget-5\%}} & \multicolumn{4}{c}{\textbf{Forget-10\%}} \\
\cmidrule(r{3pt}){2-5}\cmidrule(l{3pt}r{3pt}){6-9}\cmidrule(l{3pt}){10-13}
& ES Re.\,$\uparrow$ & ES Un.\,$\downarrow$ & $|\mathrm{Priv.\ Leak}|$\,$\downarrow$ & MU\,$\uparrow$ & ES Re.\,$\uparrow$ & ES Un.\,$\downarrow$ & $|\mathrm{Priv.\ Leak}|$\,$\downarrow$ & MU\,$\uparrow$ & ES Re.\,$\uparrow$ & ES Un.\,$\downarrow$ & $|\mathrm{Priv.\ Leak}|$\,$\downarrow$ & MU\,$\uparrow$ \\
\midrule
\multicolumn{13}{c}{\textbf{Llama-3.2-1B-Instruct}} \\
\midrule
Vanilla & 0.055 & 0.058 & 7.79 & 0.281 & 0.055 & 0.059 & 8.09 & 0.281 & 0.055 & 0.055 & 10.43 & 0.281 \\
Fully Fine-tuned & 0.736 & 0.743 & 100.00 & 0.598 & 0.736 & 0.730 & 99.99 & 0.598 & 0.736 & 0.712 & 99.46 & 0.598 \\
\midrule
NPO & 0.583 & 0.150 & 76.62 & 0.588 & 0.131 & 0.071 & \textbf{12.78} & 0.437 & 0.205 & 0.075 & \textbf{1.99} & 0.528 \\
\rowcolor{sageblue} NPO w. SAGE & \textbf{0.630\gain{0.047}} & \textbf{0.144} & \textbf{75.09} & \textbf{0.593} & \textbf{0.180\gain{0.049}} & \textbf{0.065} & 26.70 & \textbf{0.495} & \textbf{0.267\gain{0.062}} & \textbf{0.072} & 14.05 & \textbf{0.550}\\
SimNPO & 0.672 & 0.415 & \textbf{98.11} & 0.593 & 0.580 & 0.211 & 96.77 & 0.578 & 0.637 & 0.155 & 95.12 & 0.586 \\
\rowcolor{sageblue} SimNPO w. SAGE & \textbf{0.706\gain{0.034}} & \textbf{0.402} & 98.58 & \textbf{0.596} & \textbf{0.656\gain{0.076}} & \textbf{0.201} & \textbf{95.94} & \textbf{0.590} & \textbf{0.647\gain{0.010}} & \textbf{0.148} & \textbf{92.48} & \textbf{0.593}\\
RMU & 0.216 & 0.032 & \textbf{46.28} & 0.526 & 0.663 & 0.033 & \textbf{49.24} & 0.585 & 0.707 & 0.033 & 58.81 & 0.592 \\
\rowcolor{sageblue} RMU w. SAGE & \textbf{0.691\gain{0.475}} & \textbf{0.031} & 84.53 & \textbf{0.593} & \textbf{0.707\gain{0.043}} & \textbf{0.033} & 51.12 & \textbf{0.592} & \textbf{0.709\gain{0.002}} & \textbf{0.033} & \textbf{55.37} & \textbf{0.595}\\
UNDIAL & 0.484 & 0.091 & 86.42 & 0.582 & 0.232 & 0.053 & 91.46 & 0.556 & 0.268 & 0.044 & 93.50 & 0.563 \\
\rowcolor{sageblue} UNDIAL w. SAGE & \textbf{0.664\gain{0.180}} & \textbf{0.080} & \textbf{75.80} & \textbf{0.597} & \textbf{0.660\gain{0.428}} & \textbf{0.052} & \textbf{40.60} & \textbf{0.596} & \textbf{0.625\gain{0.358}} & \textbf{0.044} & \textbf{49.35} & \textbf{0.597}\\
SatImp & 0.677 & 0.206 & 95.99 & 0.596 & 0.413 & 0.069 & 76.26 & 0.572 & 0.358 & 0.059 & 72.00 & 0.552 \\
\rowcolor{sageblue} SatImp w. SAGE & \textbf{0.700\gain{0.023}} & \textbf{0.198} & \textbf{95.63} & \textbf{0.600} & \textbf{0.580\gain{0.167}} & \textbf{0.065} & \textbf{67.41} & \textbf{0.584} & \textbf{0.453\gain{0.095}} & \textbf{0.057} & \textbf{58.35} & \textbf{0.569}\\

WGA & 0.683 & 0.141 & \textbf{86.78} & 0.598 & 0.555 & 0.034 & \textbf{53.62} & \textbf{0.592} & 0.627 & 0.033 & 60.71 & 0.590 \\
\rowcolor{sageblue} WGA w. SAGE & \textbf{0.721\gain{0.038}} & \textbf{0.124} & 88.90 & \textbf{0.601} & \textbf{0.663\gain{0.107}} & \textbf{0.034} & 54.40 & 0.591 & \textbf{0.734\gain{0.108}} & \textbf{0.033} & \textbf{58.76} & \textbf{0.600}\\
\midrule
\multicolumn{13}{c}{\textbf{Llama-3.2-3B-Instruct}} \\
\midrule
Vanilla & 0.063 & 0.055 & 18.64 & 0.272 & 0.063 & 0.055 & 11.51 & 0.272 & 0.063 & 0.053 & 13.76 & 0.272 \\
Fully Fine-tuned & 0.885 & 0.920 & 100.00 & 0.665 & 0.885 & 0.887 & 100.00 & 0.665 & 0.885 & 0.888 & 99.72 & 0.665 \\
\midrule
NPO & 0.760 & 0.201 & 80.93 & 0.667 & 0.140 & 0.060 & \textbf{20.52} & 0.466 & 0.132 & 0.060 & \textbf{16.62} & 0.529 \\
\rowcolor{sageblue} NPO w. SAGE & \textbf{0.821\gain{0.062}} & \textbf{0.188} & \textbf{77.68} & \textbf{0.669} & \textbf{0.543\gain{0.403}} & \textbf{0.057} & 30.72 & \textbf{0.648} & \textbf{0.521\gain{0.389}} & \textbf{0.060} & 21.52 & \textbf{0.649}\\
SimNPO & 0.839 & 0.562 & \textbf{99.86} & 0.652 & 0.624 & 0.264 & 97.99 & 0.640 & 0.583 & 0.212 & 96.76 & \textbf{0.646} \\
\rowcolor{sageblue} SimNPO w. SAGE & \textbf{0.856\gain{0.017}} & \textbf{0.562} & \textbf{99.86} & \textbf{0.653} & \textbf{0.712\gain{0.088}} & \textbf{0.256} & \textbf{96.12} & \textbf{0.652} & \textbf{0.625\gain{0.042}} & \textbf{0.206} & \textbf{93.99} & 0.645\\
RMU & 0.320 & 0.038 & \textbf{25.00} & 0.612 & 0.830 & 0.033 & 54.30 & \textbf{0.669} & \textbf{0.859} & 0.033 & 62.83 & \textbf{0.675} \\
\rowcolor{sageblue} RMU w. SAGE & \textbf{0.817\gain{0.498}} & \textbf{0.037} & 70.90 & \textbf{0.665} & \textbf{0.860\gain{0.030}} & \textbf{0.033} & \textbf{49.85} & 0.668 & 0.854 & \textbf{0.033} & \textbf{61.06} & 0.666\\
UNDIAL & 0.750 & 0.084 & 91.24 & 0.676 & 0.304 & 0.051 & 84.27 & 0.641 & 0.357 & 0.041 & 89.69 & 0.656 \\
\rowcolor{sageblue} UNDIAL w. SAGE & \textbf{0.831\gain{0.081}} & \textbf{0.076} & \textbf{84.75} & \textbf{0.679} & \textbf{0.805\gain{0.501}} & \textbf{0.048} & \textbf{64.28} & \textbf{0.680} & \textbf{0.750\gain{0.393}} & \textbf{0.041} & \textbf{34.68} & \textbf{0.678}\\
SatImp & 0.847 & 0.348 & \textbf{96.05} & 0.659 & 0.537 & 0.091 & 55.16 & 0.614 & 0.462 & 0.048 & 38.29 & 0.614 \\
\rowcolor{sageblue} SatImp w. SAGE & \textbf{0.861\gain{0.014}} & \textbf{0.322} & 96.19 & \textbf{0.661} & \textbf{0.714\gain{0.178}} & \textbf{0.087} & \textbf{53.12} & \textbf{0.648} & \textbf{0.598\gain{0.136}} & \textbf{0.048} & \textbf{36.91} & \textbf{0.641}\\

WGA & 0.841 & 0.206 & 88.28 & 0.663 & 0.623 & 0.033 & \textbf{53.24} & 0.641 & 0.633 & 0.036 & \textbf{58.91} & 0.648 \\
\rowcolor{sageblue} WGA w. SAGE & \textbf{0.864\gain{0.023}} & \textbf{0.200} & \textbf{86.58} & \textbf{0.665} & \textbf{0.730\gain{0.107}} & \textbf{0.033} & 53.65 & \textbf{0.650} & \textbf{0.744\gain{0.111}} & \textbf{0.035} & 59.11 & \textbf{0.652}\\
\midrule
\multicolumn{13}{c}{\textbf{Llama-3.1-8B-Instruct}} \\
\midrule
Vanilla & 0.062 & 0.062 & 3.38 & 0.275 & 0.062 & 0.060 & 11.67 & 0.275 & 0.062 & 0.056 & 11.48 & 0.275 \\
Fully Fine-tuned & 0.992 & 0.977 & 100.00 & 0.627 & 0.992 & 0.972 & 100.00 & 0.627 & 0.992 & 0.979 & 99.94 & 0.627 \\
\midrule
NPO & 0.866 & 0.152 & 68.25 & \textbf{0.642} & 0.176 & 0.064 & \textbf{39.71} & 0.556 & 0.233 & 0.071 & 40.45 & 0.601 \\
\rowcolor{sageblue} NPO w. SAGE & \textbf{0.956\gain{0.090}} & \textbf{0.147} & \textbf{64.37} & 0.639 & \textbf{0.601\gain{0.424}} & \textbf{0.064} & 41.15 & \textbf{0.609} & \textbf{0.543\gain{0.310}} & \textbf{0.067} & \textbf{39.30} & \textbf{0.613}\\
SimNPO & 0.939 & 0.566 & \textbf{97.50} & 0.615 & 0.699 & 0.303 & 97.66 & \textbf{0.619} & 0.568 & 0.230 & 97.00 & 0.602 \\
\rowcolor{sageblue} SimNPO w. SAGE & \textbf{0.967\gain{0.029}} & \textbf{0.566} & 98.37 & \textbf{0.619} & \textbf{0.833\gain{0.134}} & \textbf{0.276} & \textbf{94.10} & 0.614 & \textbf{0.680\gain{0.112}} & \textbf{0.223} & \textbf{94.30} & \textbf{0.604}\\
RMU & 0.833 & 0.074 & 55.37 & \textbf{0.635} & \textbf{0.985} & 0.039 & \textbf{50.85} & \textbf{0.657} & \textbf{0.986} & 0.033 & \textbf{59.52} & \textbf{0.652} \\
\rowcolor{sageblue} RMU w. SAGE & \textbf{0.986\gain{0.152}} & \textbf{0.061} & \textbf{3.13} & 0.630 & 0.981 & \textbf{0.037} & 51.22 & 0.626 & 0.979 & \textbf{0.033} & 60.45 & 0.626\\
UNDIAL & 0.819 & 0.112 & 85.75 & \textbf{0.688} & 0.508 & 0.051 & 79.38 & \textbf{0.690} & 0.631 & 0.051 & 92.15 & \textbf{0.689} \\
\rowcolor{sageblue} UNDIAL w. SAGE & \textbf{0.976\gain{0.157}} & \textbf{0.109} & \textbf{84.87} & 0.669 & \textbf{0.955\gain{0.447}} & \textbf{0.051} & \textbf{48.78} & 0.651 & \textbf{0.932\gain{0.301}} & \textbf{0.050} & \textbf{55.31} & 0.644\\
SatImp & 0.950 & 0.461 & \textbf{94.62} & 0.618 & 0.690 & 0.117 & 84.81 & \textbf{0.627} & 0.578 & 0.046 & \textbf{5.66} & \textbf{0.596} \\
\rowcolor{sageblue} SatImp w. SAGE & \textbf{0.972\gain{0.022}} & \textbf{0.461} & 96.25 & \textbf{0.622} & \textbf{0.852\gain{0.161}} & \textbf{0.109} & \textbf{77.57} & 0.611 & \textbf{0.736\gain{0.157}} & \textbf{0.045} & 23.62 & 0.596\\
WGA & 0.946 & 0.124 & 63.75 & \textbf{0.639} & 0.614 & 0.033 & 54.85 & 0.599 & 0.537 & 0.033 & 57.99 & 0.593 \\
\rowcolor{sageblue} WGA w. SAGE & \textbf{0.962\gain{0.016}} & \textbf{0.124} & \textbf{61.62} & 0.634 & \textbf{0.792\gain{0.178}} & \textbf{0.033} & \textbf{54.20} & \textbf{0.603} & \textbf{0.809\gain{0.272}} & \textbf{0.033} & \textbf{56.59} & \textbf{0.609}\\
\bottomrule
\end{tabular}
\end{adjustbox}
\end{table*}

\paragraph{Datasets and Baselines.}

We evaluate SAGE on the OpenUnlearning benchmark~\citep{dorna2025openunlearning}, focusing primarily on TOFU~\citep{maini2024TOFU}, a fine-grained benchmark with 4,000 QA pairs for fictitious author profiles. We use the official scaling splits with different forget-set sizes (Forget-1\%, Forget-5\%, and Forget-10\%) and report main results on Llama-3-{1B, 3B, 8B}-Instruct~\citep{grattafiori2024llama3herd}. Besides, we report results on MUSE~\citep{shi2024MUSE}, which evaluates memorization and unlearning of books and news articles through verbatim reproduction, question answering, and membership inference with Llama-2-7B~\citep{touvron2023llama}, and on WMDP~\citep{li2024wmdp}, an alignment-oriented benchmark of 3,668 hazardous-domain (biosecurity, cybersecurity, chemical security) multiple-choice questions with Zephyr-7B~\citep{tunstall2023zephyr}. In all cases, SAGE is applied \emph{post-hoc} to final update vectors produced by a source unlearning method under a unified training budget of 10 epochs, including Gradient Ascent~\citep{maini2024TOFU}, GradDiff~\citep{maini2024TOFU}, NPO~\citep{fan2025SimNPO}, SimNPO~\citep{fan2025SimNPO}, RMU~\citep{li2024wmdp}, UNDIAL~\citep{dong2024UNDIAL}, CEU~\citep{yang2025CEU}, SatImp~\citep{yang2025SatImp}, WGA~\citep{wang2025wga}, DPO~\citep{rafailov2024DPO} and PDU~\citep{entesari2025PDU}.

\paragraph{Evaluation Metrics.}

We report four aspects of unlearning quality: forgetting, retention, privacy, and utility, measured by Unlearn/Retain Extraction Strength~\citep{nicholas2021extracting}, Retain ROUGE~\citep{lin2004rouge}, Privacy Leakage~\citep{shi2024MUSE}, and Model Utility~\citep{maini2024TOFU}, respectively. 
Extraction Strength quantifies the residual recoverability of target information.
Model Utility is a composite measure of overall retained capability. On TOFU, it combines probability, ROUGE, and Truth Ratio evaluations, while on MUSE and WMDP, it is reflected through retained-task generation quality.
For Privacy Leakage, we report the absolute value of the raw leakage score and evaluate methods by its magnitude only; lower values indicate a weaker forget-set membership signal and therefore better privacy protection.

\subsection{Main Results}

\paragraph{Performance on TOFU.}

Table~\ref{tab:main_tofu} shows that SAGE consistently improves retained capabilities with a better forget performance on TOFU across model scales, forget ratios, and source baselines. Averaged over TOFU settings, SAGE increases Retain E.S. from 0.5869 to 0.7409, corresponding to a 26.3\% relative improvement, while slightly reducing Unlearn E.S. from 0.1271 to 0.1227. These results indicate that post-hoc sanitization can substantially recover retained capability without weakening forgetting strength. SAGE also improves model utility from 0.609 to 0.623, and the average absolute privacy leakage decreases from 68.29 to 64.06, suggesting that suppressing retain-sensitive components in the final vectors can also mitigate collateral utility and privacy degradation.

The improvement is consistent across both model scale and forget difficulty. Across the 1B, 3B, and 8B models, the average Retain E.S. gains are 26.5\%, 29.4\%, and 23.5\%, respectively, showing that the benefit of SAGE persists from small to larger backbones. Across forget ratios, the average absolute Retain E.S. improvement is +0.109 on Forget-1\%, +0.196 on Forget-5\%, and +0.158 on Forget-10\%, indicating that SAGE remains effective under both milder and stronger unlearning regimes. Overall, these results support SAGE as a robust post-hoc correction layer that consistently improves retained capability while preserving or slightly improving forgetting quality.

\begin{table*}[t]
\centering
\small
\definecolor{sageblue}{RGB}{235,243,248}
\definecolor{sagesup}{RGB}{230,138,30}
\newcommand{\gain}[1]{\textsuperscript{\textcolor{sagesup}{\bfseries$\boldsymbol{\uparrow}\mkern-9mu$#1}}}
\newcommand{\drop}[1]{\textsuperscript{\textcolor{sagesup}{\bfseries$\boldsymbol{\downarrow}\mkern-9mu$#1}}}
\setlength{\tabcolsep}{4.2pt}
\renewcommand{\arraystretch}{1.10}
\caption{MUSE results on \textbf{Llama-2-7b-hf} with unlearn extraction strength (ES Un.), retain ROUGE (ROUGE Re.), and absolute privacy leakage ($|\mathrm{Priv.\ Leak}|$). WMDP results on \textbf{zephyr-7b-beta} are reported with unlearn accuracy (Un. acc.) and MMLU accuracy (MMLU acc.). Better results with SAGE are highlighted in \textbf{bold}, and retain-capability gains are shown in \textbf{\textcolor{sagesup}{orange}}.}
\label{tab:muse_wmdp_results}

\begin{adjustbox}{max width=\textwidth}
\begin{tabular}{lccc!{\vrule width 0.6pt}ccc!{\vrule width 0.6pt}cc}
\toprule
\multirow{2}{*}{\textbf{Method}}
& \multicolumn{3}{c}{\textbf{MUSE Books}}
& \multicolumn{3}{c}{\textbf{MUSE News}}
& \multicolumn{2}{c}{\textbf{WMDP cyber}} \\
\cmidrule(r{3pt}){2-4}
\cmidrule(l{3pt}r{3pt}){5-7}
\cmidrule(l{3pt}){8-9}
& ES Un.\,$\downarrow$ & ROUGE Re.\,$\uparrow$ & $|\mathrm{Priv.\ Leak}|$\,$\downarrow$
& ES Un.\,$\downarrow$ & ROUGE Re.\,$\uparrow$ & $|\mathrm{Priv.\ Leak}|$\,$\downarrow$
& Un. acc.\,$\downarrow$ & MMLU acc.\,$\uparrow$ \\

\cmidrule(r{3pt}){1-7}
\cmidrule(l{3pt}){8-9}
\multicolumn{1}{c}{} &
\multicolumn{6}{c!{\vrule width 0.6pt}}{\textbf{Llama-2-7b-hf}} &
\multicolumn{2}{c}{\textbf{zephyr-7b-beta}} \\
\cmidrule(r{3pt}){1-7}
\cmidrule(l{3pt}){8-9}

Vanilla
& 0.010 & 0.680 & 8.16
& 0.020 & 0.560 & 4.72
& 0.445 & 0.585 \\
Fully Fine-tuned
& 0.920 & 0.690 & 57.34
& 0.290 & 0.550 & 99.81
& -- & -- \\
\midrule


GradDiff
& 0.008 & 0.052 & 28.62
& 0.071 & 0.487 & 98.68
& 0.245 & 0.536 \\
\rowcolor{sageblue} GradDiff w. SAGE
& \textbf{0.008} & \textbf{0.610\gain{0.558}} & 32.54
& \textbf{0.070} & \textbf{0.496\gain{0.009}} & \textbf{97.08}
& \textbf{0.244} & \textbf{0.561\gain{0.025}} \\


SimNPO
& 0.138 & 0.537 & 54.81
& 0.214 & 0.486 & \textbf{99.87}
& 0.418 & 0.572 \\
\rowcolor{sageblue} SimNPO w. SAGE
& \textbf{0.106} & \textbf{0.611\gain{0.074}} & \textbf{54.66}
& \textbf{0.212} & \textbf{0.532\gain{0.047}} & \textbf{99.87}
& 0.426 & \textbf{0.587\gain{0.015}} \\

RMU
& 0.008 & 0.124 & \textbf{19.91}
& 0.021 & 0.482 & \textbf{25.59}
& \textbf{0.261} & 0.511 \\
\rowcolor{sageblue} RMU w. SAGE
& \textbf{0.008} & \textbf{0.336\gain{0.212}} & 38.76
& \textbf{0.019} & \textbf{0.496\gain{0.014}} & 30.90
& 0.279 & \textbf{0.570\gain{0.059}} \\

UNDIAL
& 0.024 & 0.627 & 18.45
& 0.013 & 0.189 & 99.45
& \textbf{0.390} & 0.565 \\
\rowcolor{sageblue} UNDIAL w. SAGE
& \textbf{0.024} & \textbf{0.700\gain{0.073}} & \textbf{18.05}
& \textbf{0.013} & \textbf{0.264\gain{0.075}} & \textbf{89.86}
& 0.397 & \textbf{0.584\gain{0.019}} \\



WGA
& 0.008 & 0.486 & 43.82
& 0.012 & 0.436 & 103.70
& \textbf{0.348} & 0.546 \\
\rowcolor{sageblue} WGA w. SAGE
& \textbf{0.008} & \textbf{0.619\gain{0.133}} & \textbf{25.55}
& \textbf{0.011} & \textbf{0.468\gain{0.032}} & \textbf{102.60}
& \textbf{0.348} & \textbf{0.570\gain{0.024}} \\

PDU
& \textbf{0.008} & 0.042 & \textbf{54.48}
& 0.146 & \textbf{0.501} & \textbf{99.79}
& \textbf{0.243} & 0.243 \\
\rowcolor{sageblue} PDU w. SAGE
& \textbf{0.008} & \textbf{0.605\gain{0.563}} & 71.21
& \textbf{0.141} & 0.484 & 99.81
& 0.245 & \textbf{0.255\gain{0.012}} \\

\bottomrule
\end{tabular}
\end{adjustbox}
\end{table*}

\paragraph{Performance on MUSE and WMDP.}
MUSE involves longer-form Books and News content with open-ended generation behavior, while WMDP-cyber evaluates hazardous knowledge removal together with general capability preservation. 
On MUSE, SAGE improves retention capabilities while preserving forgetting: averaged over Books and News, ROUGE Re. increases from 0.371 to 0.518, ES Un. slightly decreases from 0.056 to 0.052. 
On WMDP-cyber, SAGE mainly improves the utility side of the safety--utility trade-off, increasing average MMLU accuracy from 0.496 to 0.521 across source baselines at comparable unlearning accuracy.
These results show that SAGE remains effective on longer-form and safety-oriented unlearning tasks.

\subsection{Ablation Study}

\begin{table*}[t]
\centering
\scriptsize
\definecolor{sageblue}{RGB}{235,243,248}
\newcommand{\best}[1]{\textbf{#1}}
\newcommand{\second}[1]{\underline{#1}}
\newcommand{\sagecell}[1]{\cellcolor{sageblue}#1}
\setlength{\tabcolsep}{2.8pt}
\renewcommand{\arraystretch}{1.03}
\caption{Hyperparameter ablations of SAGE on \textbf{Llama-3-1B} under TOFU Forget-10\% Split. Default parameters are shaded in blue. Best and second-best results are highlighted in \textbf{bold} and \underline{underlined}.}
\label{tab:ablation_compact}
\begin{adjustbox}{max width=\textwidth,center}
\begin{tabular}{lccccccccc}
\toprule
\multirow{2}{*}{\textbf{Ablation}} 
& \multirow{2}{*}{\textbf{Setting}} 
& \multicolumn{4}{c}{\textbf{WGA w. SAGE}} 
& \multicolumn{4}{c}{\textbf{RMU w. SAGE}} \\
\cmidrule(lr){3-6} \cmidrule(lr){7-10}
& & ES Re.\,$\uparrow$ & ES Un.\,$\downarrow$ & $|\mathrm{Priv.\ Leak}|$\,$\downarrow$ & MU\,$\uparrow$
   & ES Re.\,$\uparrow$ & ES Un.\,$\downarrow$ & $|\mathrm{Priv.\ Leak}|$\,$\downarrow$ & MU\,$\uparrow$ \\
\midrule

\multirow{5}{*}{SVD Trunc.\ $\rho$} 
& 0.5
& 0.677 & 0.033 & 59.49 & 0.594
& 0.679 & 0.033 & 60.01 & 0.590 \\
& 0.7
& 0.699 & 0.033 & 59.24 & 0.596
& 0.696 & 0.033 & 59.89 & 0.591 \\
& \sagecell{\textbf{0.9}}
& \sagecell{\second{0.734}} & \sagecell{0.033} & \sagecell{\best{58.76}} & \sagecell{\second{0.600}}
& \sagecell{0.709} & \sagecell{0.033} & \sagecell{\best{55.37}} & \sagecell{\second{0.595}} \\
& 0.95
& \best{0.748} & 0.033 & \second{58.40} & \best{0.603}
& \second{0.716} & 0.033 & \second{55.48} & \best{0.595} \\
& 1.0
& 0.607 & 0.033 & 60.51 & 0.591
& \best{0.726} & 0.033 & 58.53 & 0.594 \\
\midrule

\multirow{4}{*}{Retain-proxy Size} 
& 16
& 0.713 & 0.033 & 59.44 & 0.596
& 0.699 & 0.033 & 59.06 & 0.592 \\
& 64
& 0.724 & 0.033 & 58.84 & 0.599
& 0.705 & 0.033 & 55.88 & \second{0.593} \\
& \sagecell{\textbf{128}}
& \sagecell{\second{0.734}} & \sagecell{0.033} & \sagecell{\second{58.76}} & \sagecell{\second{0.600}}
& \sagecell{\second{0.709}} & \sagecell{0.033} & \sagecell{\second{55.37}} & \sagecell{\best{0.595}} \\
& 256
& \best{0.764} & 0.033 & \best{58.37} & \best{0.603}
& \best{0.716} & 0.033 & \best{55.30} & \best{0.595} \\
\midrule

\multirow{6}{*}{Regularization $\lambda$} 
& 0.0
& 0.629 & 0.033 & 60.66 & 0.592
& 0.706 & 0.033 & 58.86 & 0.592 \\
& 0.01
& 0.722 & 0.033 & 58.94 & 0.599
& 0.704 & 0.033 & 56.15 & 0.592 \\
& 0.1
& 0.730 & 0.033 & 58.97 & \best{0.603}
& \best{0.712} & 0.033 & \best{55.16} & \second{0.594} \\
& 1.0
& \second{0.734} & 0.033 & 58.78 & \second{0.601}
& 0.707 & 0.033 & 55.57 & 0.592 \\
& \sagecell{\textbf{5.0}}
& \sagecell{\second{0.734}} & \sagecell{0.033} & \sagecell{\best{58.76}} & \sagecell{0.600}
& \sagecell{0.709} & \sagecell{0.033} & \sagecell{\second{55.37}} & \sagecell{\best{0.595}} \\
& 10.0
& \best{0.736} & 0.033 & \second{58.76} & 0.600
& \second{0.709} & 0.033 & 55.39 & 0.594 \\
\midrule

\multirow{5}{*}{Scaling Coefficient $\alpha$} 
& 0.8
& \best{0.781} & 0.046 & \best{5.23} & \best{0.602}
& \best{0.719} & 0.041 & \second{54.42} & \best{0.595} \\
& 1.0
& \second{0.769} & \second{0.037} & \second{38.07} & \second{0.601}
& \second{0.709} & \second{0.033} & 55.37 & \second{0.595} \\
& 1.5
& 0.746 & \best{0.033} & 58.10 & 0.601
& 0.694 & \best{0.033} & 59.01 & 0.588 \\
& 3.0
& 0.649 & \best{0.033} & 55.64 & 0.596
& 0.591 & \best{0.033} & 55.41 & 0.571 \\
& 6.0
& 0.389 & \best{0.033} & 54.52 & 0.553
& 0.370 & \best{0.033} & \best{30.19} & 0.496 \\
\bottomrule
\end{tabular}
\end{adjustbox}
\end{table*}

\paragraph{SVD Trunc. $\rho$.}

As the truncation ratio $\rho$ increases, retain-side performance generally improves because a more dominant retained activation geometry is preserved. However, values too close to full rank can over-constrain sanitization and hurt stability, privacy or utility; overall, $\rho=0.9$ provides the best balance.

\paragraph{Retain-proxy Size.}

SAGE is relatively insensitive to retain-proxy size: even small proxies already recover a useful estimate of the dominant retained geometry and yield consistent gains. We therefore use 128 retained examples by default.

\paragraph{Regularization $\lambda$.}

The regularization strength $\lambda$ controls the trade-off between suppressing retain-sensitive directions and staying close to the unlearning vectors. On TOFU, performance improves quickly from $\lambda=0$ and then saturates, motivating our default choice $\lambda=5$; on MUSE and WMDP, smaller values may work better.

\paragraph{Scaling Coefficient $\alpha$.}

The scaling coefficient $\alpha$ controls how much of the sanitized update is restored before being added back to the base model. As $\alpha$ increases, forgetting and utility often improve, but overly large values can reintroduce retain-side interference and reduce retained performance.

\subsection{Integration with Other Plug-in Methods}

\begin{table*}[t]
\centering
\scriptsize
\definecolor{sageblue}{RGB}{235,243,248}
\definecolor{sagesup}{RGB}{230,138,30}
\newcommand{\gain}[1]{\textsuperscript{\textcolor{sagesup}{\bfseries$\boldsymbol{\uparrow}\mkern-9mu$#1}}}
\newcommand{\drop}[1]{\textsuperscript{\textcolor{sagesup}{\bfseries$\boldsymbol{\downarrow}\mkern-9mu$#1}}}
\newcommand{\best}[1]{\textbf{#1}}
\newcommand{\second}[1]{\underline{#1}}
\setlength{\tabcolsep}{4.4pt}
\renewcommand{\arraystretch}{1.08}
\caption{Results of SAGE with GU and GRU plug-ins for WGA, NPO, and DPO on TOFU over Forget-1\%, 5\%, 10\% using \textbf{Llama-3-1B}. w. SAGE are shaded in blue, the Best and second best results are highlighted in \textbf{bold} and \underline{underline}, and additional gains brought by SAGE are in \textbf{\textcolor{sagesup}{orange}}.}
\label{tab:app:plugin_compositionality_full}

\begin{adjustbox}{max width=\textwidth,center}
\begin{tabular}{l!{\vrule width 0.6pt}cccc!{\vrule width 0.6pt}cccc!{\vrule width 0.6pt}cccc}
\toprule
\multirow{2}{*}{\textbf{Method}}
& \multicolumn{4}{c!{\vrule width 0.6pt}}{\textbf{WGA}}
& \multicolumn{4}{c!{\vrule width 0.6pt}}{\textbf{NPO}}
& \multicolumn{4}{c}{\textbf{DPO}} \\
\cmidrule(lr){2-5}
\cmidrule(lr){6-9}
\cmidrule(lr){10-13}
& ES Re.\,$\uparrow$ & ES Un.\,$\downarrow$ & $|\mathrm{Priv.\ Leak}|$\,$\downarrow$ & MU\,$\uparrow$
& ES Re.\,$\uparrow$ & ES Un.\,$\downarrow$ & $|\mathrm{Priv.\ Leak}|$\,$\downarrow$ & MU\,$\uparrow$
& ES Re.\,$\uparrow$ & ES Un.\,$\downarrow$ & $|\mathrm{Priv.\ Leak}|$\,$\downarrow$ & MU\,$\uparrow$ \\
\midrule

\multicolumn{13}{c}{\textbf{Forget-1\%}} \\
\midrule
Baseline
& \second{0.683} & \best{0.141} & \best{86.78} & 0.598
& 0.583 & 0.150 & 76.62 & 0.588
& 0.575 & 0.369 & 98.94 & 0.575 \\

\rowcolor{sageblue}
w. SAGE
& \best{0.721} & \best{0.124} & \second{88.90} & \second{0.601}
& \second{0.630} & \best{0.144} & \second{75.09} & \second{0.593}
& \second{0.631} & \best{0.327} & \second{97.87} & \second{0.583} \\

w. GU
& 0.677 & 0.179 & 92.80 & \best{0.603}
& 0.548 & 0.178 & 81.35 & 0.590
& 0.577 & 0.392 & 99.06 & 0.576 \\

\rowcolor{sageblue}
w. GU w. SAGE
& 0.691\gain{0.014} & 0.188 & 91.03 & 0.600
& 0.605\gain{0.056} & 0.160 & \best{74.97} & 0.591
& \best{0.634}\gain{0.057} & \second{0.349} & \best{97.76} & \best{0.584} \\

w. GRU
& 0.681 & \best{0.141} & \best{86.78} & 0.599
& 0.584 & 0.150 & 76.86 & 0.588
& 0.574 & 0.362 & 98.94 & 0.575 \\

\rowcolor{sageblue}
w. GRU w. SAGE
& \second{0.715}\gain{0.034} & \second{0.145} & 89.61 & 0.601
& \best{0.635}\gain{0.050} & \second{0.148} & 77.92 & \best{0.593}
& 0.630\gain{0.056} & \best{0.327} & \second{97.87} & 0.583 \\
\midrule

\multicolumn{13}{c}{\textbf{Forget-5\%}} \\
\midrule
Baseline
& 0.555 & 0.034 & \second{53.62} & \best{0.592}
& 0.131 & 0.071 & \second{12.78} & 0.437
& 0.235 & 0.150 & 84.12 & \second{0.017} \\

\rowcolor{sageblue}
w. SAGE
& \best{0.663} & \second{0.034} & 54.40 & \second{0.591}
& \second{0.180} & \best{0.065} & 26.70 & \second{0.495}
& \second{0.349} & \best{0.140} & \second{71.31} & 0.013 \\

w. GU
& 0.504 & 0.037 & \best{49.20} & 0.584
& 0.123 & 0.068 & \best{1.27} & 0.430
& 0.208 & 0.147 & 81.63 & 0.000 \\

\rowcolor{sageblue}
w. GU w. SAGE
& 0.482 & 0.036 & 53.69 & 0.575
& 0.161\gain{0.038} & 0.066 & 21.39 & 0.478
& 0.314\gain{0.105} & \second{0.146} & \best{70.88} & 0.000 \\

w. GRU
& 0.552 & \best{0.034} & 53.74 & 0.589
& 0.131 & 0.069 & 13.59 & 0.441
& 0.235 & 0.150 & 84.17 & \best{0.017} \\

\rowcolor{sageblue}
w. GRU w. SAGE
& \second{0.568}\gain{0.016} & \best{0.034} & 54.61 & 0.585
& \best{0.182}\gain{0.050} & \second{0.066} & 27.52 & \best{0.497}
& \best{0.357}\gain{0.123} & 0.148 & 74.19 & 0.014 \\
\midrule

\multicolumn{13}{c}{\textbf{Forget-10\%}} \\
\midrule
Baseline
& 0.627 & 0.033 & 60.71 & 0.590
& 0.205 & 0.075 & \second{1.99} & 0.528
& 0.316 & 0.192 & 93.95 & \second{0.137} \\

\rowcolor{sageblue}
w. SAGE
& \best{0.734} & 0.033 & \second{58.76} & \best{0.600}
& \best{0.267} & 0.072 & 14.05 & \second{0.550}
& \second{0.451} & 0.191 & 91.13 & 0.125 \\

w. GU
& 0.592 & 0.033 & 60.43 & 0.590
& 0.178 & \second{0.072} & 6.18 & 0.508
& 0.293 & \second{0.189} & 93.53 & 0.132 \\

\rowcolor{sageblue}
w. GU w. SAGE
& 0.651\gain{0.058} & 0.033 & \best{58.40} & 0.593
& 0.225\gain{0.048} & \best{0.071} & 16.95 & 0.540
& 0.395\gain{0.101} & \best{0.177} & \best{90.49} & 0.081 \\

w. GRU
& 0.632 & 0.033 & 60.79 & 0.593
& 0.197 & 0.075 & \best{1.41} & 0.525
& 0.318 & 0.194 & 94.00 & \best{0.137} \\

\rowcolor{sageblue}
w. GRU w. SAGE
& \second{0.681}\gain{0.049} & 0.033 & 59.09 & \second{0.594}
& \second{0.265}\gain{0.068} & 0.072 & 14.89 & \best{0.553}
& \best{0.453}\gain{0.135} & 0.191 & \second{91.10} & 0.123 \\
\bottomrule
\end{tabular}
\end{adjustbox}
\end{table*}

We further study whether SAGE is compatible with training-time rectifiers such as GU and GRU. Overall, the results show that SAGE is complementary to these gradient-level plugins: in most settings, applying SAGE on top of GU/GRU further improves retained capability while largely preserving the original forgetting behavior, and often also brings gains in utility or privacy leakage. This indicates that training-time rectification does not fully remove retain-harmful components from the final update, leaving room for post-hoc sanitization to provide additional benefits.

At the same time, using SAGE alone is generally more effective than using GU or GRU alone. Across different source methods and forget ratios, SAGE more consistently improves the retain--forget trade-off, suggesting that directly sanitizing the final deployed update can be a stronger intervention than only rectifying gradients during training. Taken together, these results show that SAGE is a strong standalone plugin and complementary post-hoc component for existing training-time methods.

\subsection{Suppression Difference Analysis}
\label{sec:sd_analysis}

To examine the mechanism discussed in Section~\ref{sec:method_sanitization}, we analyze the module-wise suppression difference \(\Gamma_g = S_g^{(r)} - S_g^{(f)}\), which measures whether SAGE suppresses retain-side response more strongly than forget-side response. As shown in Figure~\ref{fig:sd_analysis}, the distribution is consistently shifted to the positive side, which indicates that SAGE tends to remove more retain-harming responses than forget-relevant responses at the module level. Such asymmetric suppression helps explain why SAGE can improve retention while preserving forgetting.

\subsection{Computation Resources}

Figure~\ref{fig:compute_resources} reports the computational overhead of SAGE. 
On Llama-3-3B, source unlearning training takes 22--38 minutes and exceeds 100 GiB peak GPU memory across five representative baselines, whereas SAGE takes about 9 minutes on average and uses only 7.86 GiB peak GPU memory. 
Moreover, the projector cache is built once for a given base model and retain proxy and can be reused across multiple source checkpoints, so only the lightweight sanitization step are repeated when applying SAGE on the same model. 
The scaling results show that efficiency advantage persists from 1B to 8B models, making SAGE a practical post-hoc correction under limited compute budgets.

\begin{figure}[t]
    \centering
    \begin{minipage}[t]{0.57\linewidth}
        \centering
        \includegraphics[width=\linewidth]{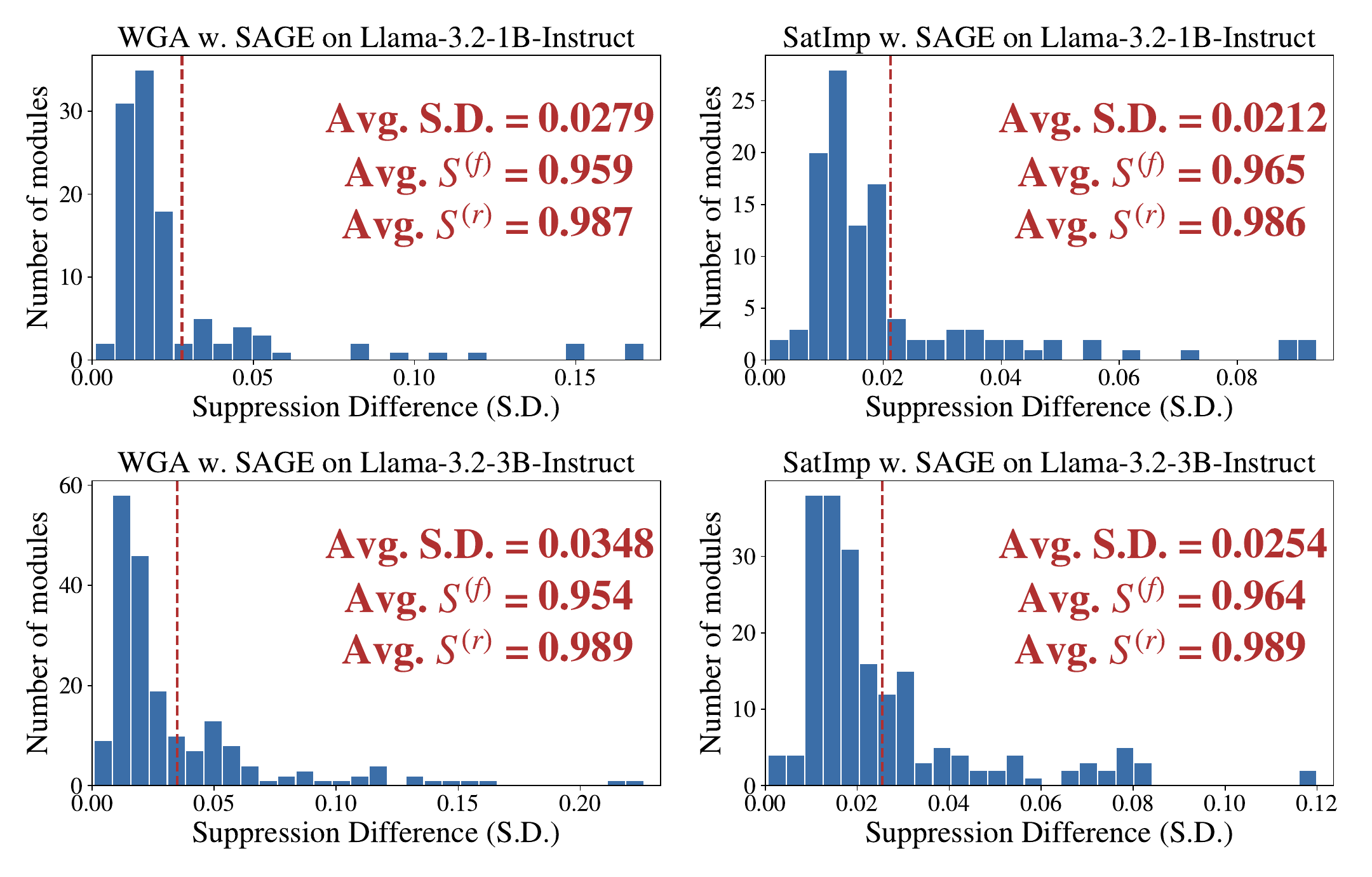}
        \caption{Distribution of suppression difference \\ where the vertical red dashed line denotes the mean.}
        \label{fig:sd_analysis}
    \end{minipage}\hfill
    \begin{minipage}[t]{0.42\linewidth}
        \centering
        \includegraphics[width=\linewidth]{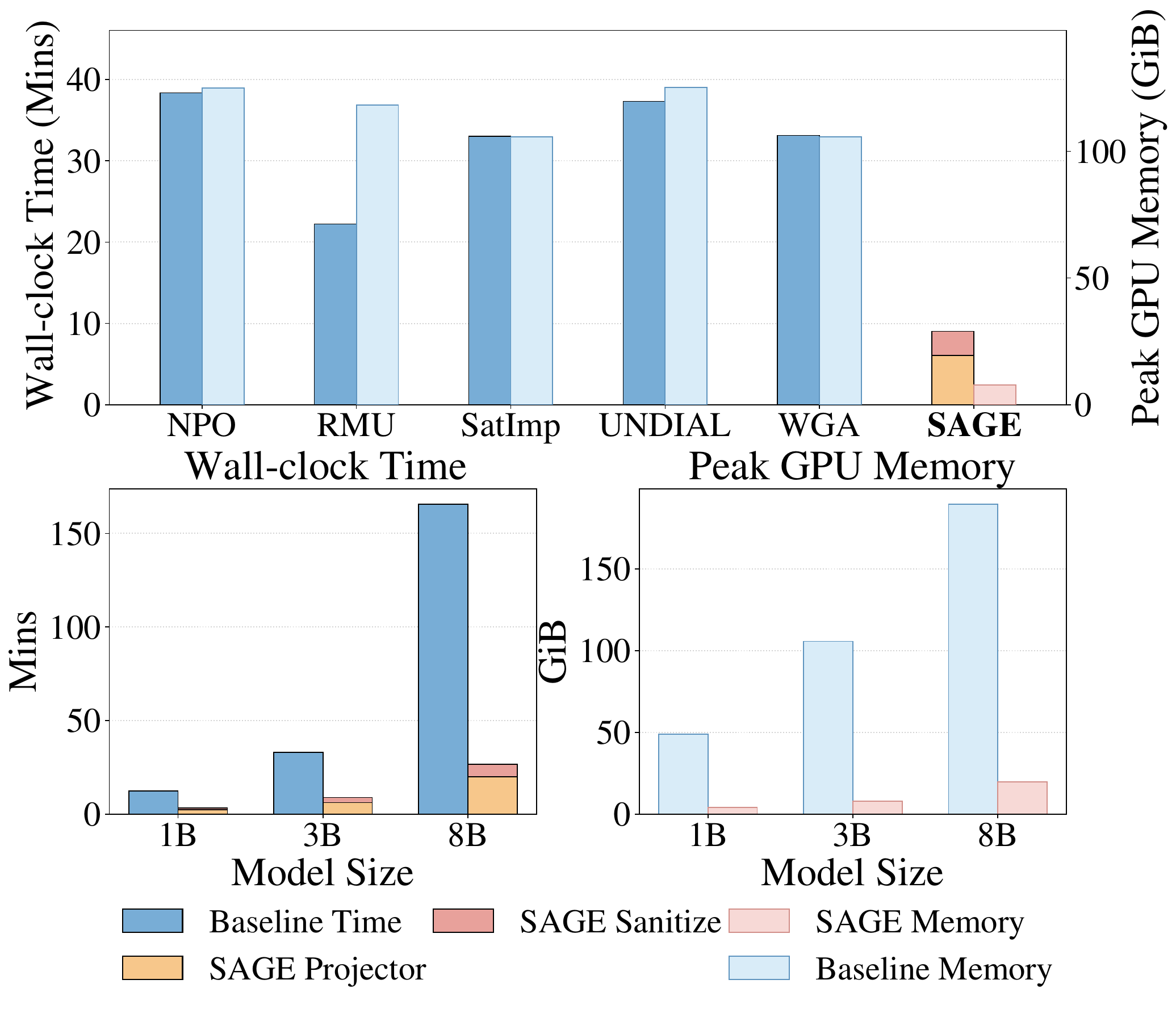}
        \caption{\textbf{Computation Resources} for Baselines and SAGE on different models.}
        \label{fig:compute_resources}
    \end{minipage}
    \vspace{-1.5\baselineskip}
    
\end{figure}

\FloatBarrier

\section{Conclusion}

In this paper, we study a practical post-hoc final-update setting for LLM unlearning. We show that source unlearning updates can still contain retain-harmful components that induce substantial drift in retained activations and degrade retained behavior. 
To address this, we propose \textbf{SAGE} that extracts retained activation geometry from a small retain proxy and applies a closed-form spectral operator to suppress update components aligned with high-energy retained directions while preserving the source method's forgetting carrier. 
Extensive experiments on TOFU, MUSE, and WMDP demonstrate that SAGE consistently improves the retain--forget trade-off, while also improving utility and privacy leakage without retraining. We further show that SAGE is complementary to training-time plug-ins. Overall, our results identify post-hoc sanitization of final unlearning updates as a practical and underexplored design axis for machine unlearning, and suggest that directly correcting the final deployed update can be an effective way to improve unlearning quality.

\newpage

{
    \small
    \bibliographystyle{ieeenat_fullname}
    \bibliography{main}
}

\end{document}